\documentclass{article}
\usepackage{microtype}
\usepackage{graphicx}
\usepackage{caption}
\usepackage{subcaption}
\usepackage{booktabs} 
\usepackage{slashbox}

\usepackage{hyperref}

\usepackage[accepted]{icml2018}

\usepackage{natbib}
\usepackage{amsmath}
\usepackage{amsthm}
\usepackage{amssymb}
\usepackage{mathtools}
\usepackage{tikz}
\usepackage{xcolor}
\usetikzlibrary{arrows}

\usepackage{algorithm}
\usepackage{algorithmic}
\usepackage{hyperref}
\usepackage{bm}

\newcommand{\argmin}{\mathrm{argmin}}
\newcommand{\argmax}{\mathrm{argmax}}

\newcommand{\rank}{\mathrm{rank}}

\newcommand{\range}{\mathrm{Range}}

\newcommand{\mat}[1]{\mathbf{#1}}
\newcommand{\vect}[1]{\mathbf{#1}}
\newcommand{\norm}[1]{\left\|#1\right\|}

\newcommand{\expect}{\mathbb{E}}
\newcommand{\prob}{\mathbb{P}}

\newcommand{\relu}[1]{\sigma\left(#1\right)}

\newcommand{\vectorize}[1]{\text{vect}\left(#1\right)}

\newcommand{\nullspace}{\mathcal{N}}

\newtheorem{thm}{Theorem}[section]

\newtheorem{cor}{Corollary}[section]

\newtheorem{property}{Property}[section]
\newtheorem{defn}{Definition}[section]

\icmltitlerunning{Overparametrization Helps Optimization and Generalizes Well}

\begin{document}
	
	\twocolumn[
	\icmltitle{On the Power of Over-parametrization in \\Neural Networks with Quadratic Activation} 
	\ifdefined\usebigfont
	\onecolumn
	\fi
	
	
	
	\begin{icmlauthorlist}
		\icmlauthor{Simon S. Du}{ml}
		\icmlauthor{Jason D. Lee}{usc}
	\end{icmlauthorlist}
	
	\icmlaffiliation{ml}{Machine Learning Department, Carnegie Mellon University}
	\icmlaffiliation{usc}{Department of Data Sciences and Operations, University of Southern California}

	\icmlcorrespondingauthor{Simon S. Du}{ssdu@cs.cmu.edu}

	\icmlkeywords{over-parameterization, generalization}
	
	\vskip 0.3in
	]
	
	
	
	\printAffiliationsAndNotice{}  

\begin{abstract}
	\label{sec:abs}
	We provide new theoretical insights on why over-parametrization is effective in learning neural networks.
For a $k$ hidden node shallow network with quadratic activation and $n$ training  data points, we show as long as $ k \ge \sqrt{2n}$, over-parametrization enables local search algorithms to find a  \emph{globally} optimal solution for general smooth and convex loss functions.
Further, despite that the number of parameters may exceed the sample size, using theory of Rademacher complexity, we show with weight decay, the solution also generalizes well if the data is sampled from a regular distribution such as Gaussian.
To prove when $k\ge \sqrt{2n}$, the loss function has benign landscape properties, we adopt an idea from smoothed analysis, which may have other applications in studying loss surfaces of neural networks.
\end{abstract}

\section{Introduction}
\label{sec:intro}
Neural networks have achieved a remarkable impact on many applications such computer vision, reinforcement learning and natural language processing.
Though neural networks are successful in practice, their theoretical properties are not yet well understood.
Specifically, there are two intriguing empirical observations that existing theories cannot explain.
\begin{itemize}
\item \textbf{Optimization}: Despite the highly non-convex nature of the objective function, simple first-order algorithms like stochastic gradient descent are able to minimize the training loss of neural networks.
Researchers have conjectured that the use of  over-parametrization~\citep{livni2014computational,safran2017spurious} is the primary reason why local search algorithms can achieve low training error.
The intuition is over-parametrization alters the loss function to have a large manifold of globally optimal solutions, which in turn allows local search algorithms to more easily find a global optimal solution.
\item \textbf{Generalization}: 
From the statistical point of view, over-parametrization may hinder effective generalization, since it greatly increases the number of parameters to the point of  having number of parameters exceed the sample size. To address this, practitioners often use explicit forms of regularization such as weight decay, dropout, or early stopping to improve generalization. However in the non-convex setting, theoretically, we do not have a good quantitative understanding on how these regularizations help generalization for neural network models.
\end{itemize} 

In this paper, we provide new theoretical insights into the optimization landscape and generalization ability of over-parametrized neural networks.
Specifically we consider the neural network of the following form:\begin{align}
f(\vect{x},\mat{W}) = \sum_{j=1}^{k}a_i\relu{\langle\vect{w}_j,\vect{x}\rangle}.\label{eqn:architecture}
\end{align}
In the above $\vect{x} \in \mathbb{R}^d$ is the input vector, $\mat{W} \in \mathbb{R}^{d \times k}$ with $\vect{w}_j \in \mathbb{R}^d$ denotes the $j$-th row of $\mat{W}$ and $a_i$'s are the weights in the second layer.
Finally $\relu{\cdot} : \mathbb{R} \rightarrow \mathbb{R}$ denotes the activation function applied to each hidden node.
When the neural network is  over-parameterized, the number of hidden notes $k$ can be very large compared with input dimension $d$ or the number of training samples.

In our setting, we fix the second layer to be $\vect{a} = \left(1,\ldots,1\right)$.
Although it is simpler than the case where the second layer is not fixed, the effect of over-parameterization can be studied in this setting as well because we do not have any restriction on the number of hidden nodes.

We focus on quadratic activation function $\relu{z}=z^2$.
Though quadratic activations are rarely used in practice, stacking multiple such two-layer blocks can be used to simulate higher-order polynomial neural networks and sigmodial activated neural networks~\citep{livni2014computational,soltanitowards}.

In practice, we have  $n$ training samples $\left\{\vect{x}_i,y_i\right\}_{i=1}^n$ and solve the following optimization problem to learn a neural network
\begin{align*}
\min_\mat{W} \frac{1}{n}\sum_{i=1}^{n}\ell\left(f(\vect{x}_i,\mat{W}),y_i\right)
\end{align*}where $\ell(\cdot,\cdot)$ is some loss function such as $\ell_2$ or logistic loss.
For gradient descent we use the following update\begin{align*}
\mat{W} \leftarrow \mat{W} - \eta\sum_{i=1}^{n} \nabla_\mat{W} \ell\left(f(x_i),y_i\right)
\end{align*} where $\eta$ is the step size.

To improve the generalization ability, we often add explicit regularization.
In this paper, we focus on a particular regularization technique, weight decay for which we slightly change the gradient descent algorithm to
\begin{align*}
\mat{W} \leftarrow \mat{W} - \eta\sum_{i=1}^{n}  \nabla_\mat{W} \ell\left(f(\vect{x}_i,\mat{W}),y_i\right) - \eta \lambda \mat{W}.
\end{align*} where $\lambda$ is the decay rate.
Note this algorithm is equivalent to applying the gradient descent algorithm on the regularized loss
\begin{align}
\min_\mat{W}L\left(\mat{W}\right) = \frac{1}{n}\sum_{i=1}^{n}\ell\left(f(\vect{x}_i),y_i\right) + \frac{\lambda}{2}\norm{\mat{W}}_F^2. \label{eqn:opt_problem}
\end{align}
In this setup, we make the following theoretical contributions to explain why over-parametrization helps optimization and still allows for generalization.

\subsection{Main Contributions}
\paragraph{Over-parametrization Helps Optimization.}
We analyze two kinds of over-parameterization.
First we show that for \[
k \ge d,
\] then all local minima in Problem~\eqref{eqn:opt_problem} is global and all saddle points are strict.
This properties together with recent algorithmic advances in non-convex optimization~\citep{lee2016gradient} imply gradient descent can find a globally optimal solution with random initialization.
This is a minor generalization of results in \cite{soltanolkotabi2017theoretical} which only includes $\ell_2$ loss, and \cite{haeffele2015global,haeffele2014structured} which only include $k\ge d+1$.

Second, we consider another form of over-parametrization, \[\frac{k(k+1)}{2} > n.\]
This condition on the amount of over-parameterization is much milder than $k \ge n$, a condition used in many previous papers~\citep{nguyen2017loss,nguyen2017loss2}.
Further in practice, $k(k+1)/2 > n$ is a much milder requirement than $k \ge d$, since if $k \approx \sqrt{2n}$ and $n << d^2$ then $ k << d$.
In this setting, we consider the \emph{perturbed} version of the Problem~\eqref{eqn:opt_problem}:
\begin{align}
\min_\mat{W}L_{\mat{C}}\left(\mat{W}\right) 
=&\frac{1}{n}\sum_{i=1}^{n}\ell\left(f(\vect{x}_i),y_i\right) +  \frac{\lambda}{2}\norm{\mat{W}}_F^2  \nonumber\\
&+\langle \mat{C}, \mat{W}^\top\mat{W}\rangle  \label{eqn:opt_problem_perturbed}
\end{align}
where $\mat{C}$ is a random positive semidefinite matrix with \emph{arbitrarily} small Frobenius norm.
We show that if $\frac{k(k+1)}{2} > n$, Problem~\eqref{eqn:opt_problem_perturbed} also has the desired properties that all local minima are global and all saddle points are strict with probability $1$.
Since $\mat{C}$ has small Frobenius norm, the optimal value of Problem~\eqref{eqn:opt_problem_perturbed} is very close to that of Problem~\eqref{eqn:opt_problem}. 
See Section~\ref{sec:opt} for the precise statement. 

To prove this surprising fact, we bring forward ideas from smoothed analysis in constructing the perturbed loss function \eqref{eqn:opt_problem_perturbed}, which we believe is useful for analyzing the landscape of non-convex losses. 
\paragraph{Weight-decay Helps Generalization.}
We show because of weight-decay, the optimal solution of Problem~\eqref{eqn:opt_problem} also generalizes well.
The major observation is weight-decay ensures the solution of Problem~\eqref{eqn:opt_problem} has low \emph{Frobenius} norm, which is equivalent to matrix $\mat{W}^\top\mat{W}$ having  low \emph{nuclear} norm \citep{srebro2005maximum}.
This observation allows us to use theory of Rademacher complexity to directly obtain quantitative generalization bounds.
Our theory applies to a wide range of data distribution and in particular, does not need to assume the model is realizable.
Further, the generalization bound does not depend on the number of epochs SGD runs or the number of hidden nodes.

To sum up, in this paper we justify the following folklore.
\begin{center}
\emph{Over-parametrization allows us to find global optima and with weight decay, the solution also generalizes well.}
\end{center}

\subsection{Organization}
\label{sec:org}
This paper is organized as follows.
In Section~\ref{sec:pre} we introduce necessary background and definitions.
In Section~\ref{sec:opt} we present our main theorems on why over-parametrization helps optimization when $k \ge d$ \emph{ or }$\frac{k(k+1)}{2} > n$.
In Section~\ref{sec:generalization}, we give quantitative generalization bounds to explain why weight decay helps generalization in the presence of over-parametrization.
In Section~\ref{sec:proof}, we prove our main theorems.
We conclude and list future works in Section~\ref{sec:con}.

\subsection{Related Works}
\label{sec:rel}

Neural networks have enjoyed great success in many practical applications~\citep{krizhevsky2012imagenet,dauphin2016language,silver2016mastering}.
To explain this success, many works have studied the expressiveness of neural networks.
The expressive ability of shallow neural network dates back to 90s~\citep{barron1994approximation}.
Recent results give more refined analysis on deeper models~\citep{bolcskei2017optimal,telgarsky2016benefits,wiatowski2017energy}.

However, from the point of view of learning theory, it is well known that training a neural network is hard in the worst case~\citep{blum1989training}.
Despite the worst-case pessimism, local search algorithms such as gradient descent are very successful in practice.
With some additional assumptions, many works tried to design algorithms that provably learn a neural network~\citep{goel2016reliably,sedghi2014provable,janzamin2015beating}.
However these algorithms are not gradient-based and do not provide insight on why local search algorithm works well.

Focusing on gradient-based algorithms, a line of research~\citep{tian2017analytical,brutzkus2017globally,zhong2017learning,zhong2017recovery,li2017convergence,du2017convolutional,du2017cnn} analyzed the behavior of (stochastic) gradient descent with a structural assumption on the input distribution. 
The major drawback of these papers is that they all focus on the regression setting with least-squares loss and further assume the model is realizable meaning the label is the output of a neural network plus a zero mean noise, which is unrealistic. In the case of more than one hidden unit, the papers of \cite{li2017convergence,zhong2017recovery} further require a stringent initialization condition to recover the true parameters.

Finding the optimal weights of a neural network is non-convex problem.
Recently, researchers found that if the objective functions satisfy the following two key properties:
	(1) all local minima are global
		and (2) all saddle points and local maxima are strict,
then first order method like gradient descent~\citep{ge2015escaping,jin2017escape,levy2016power,du2017gradient,lee2016gradient} can find a global minimum.

This motivates the research of studying the landscape of neural networks~\citep{kawaguchi2016deep,choromanska2015loss,freeman2016topology,zhou2017landscape,nguyen2017loss,nguyen2017loss2,ge2017learning,safran2017spurious,soltanolkotabi2017theoretical,poston1991local,haeffele2015global,haeffele2014structured,soudry2017exponentially}
In particular,~\citet{haeffele2015global,poston1991local,nguyen2017loss,nguyen2017loss2}studied the effect of over-parameterization on training the neural networks.
These results require a large amount of over-parameterization that the width of one of the hidden layers has to be greater than the number of training examples, which is unrealistic in commonly used neural networks.
Recently, \citet{soltanolkotabi2017theoretical} showed   for shallow neural networks, the number of hidden nodes is only required to be larger or equal to the input dimension for $\ell_2$-loss.
In comparison, our theorems work for general loss functions with regularization under the same assumption.
Further we also propose a new form of over-parameterization, namely as long as $k\ge\sqrt{2n}$, the loss function also admits a benign landscape.

We now turn our attention to generalization ability of learned neural networks.
It is well known that the classical learning theory cannot explain the generalization ability because VC-dimension of neural networks is large~\citep{harvey2017nearly,zhang2016understanding}.
A line of research tries to explain this phenomenon by studying the implicit regularization from stochastic gradient descent algorithm~\citep{hardt2015train,pensia2018generalization,mou2017generalization,brutzkus2017sgd,li2017algorithmic}.
However, the generalization bounds of these papers often depend on the number of epochs SGD runs, which is large in practice.
Another direction is to study the generalization ability based on the norms of weight matrices in neural networks~\citep{neyshabur2015norm,neyshabur2017pac,neyshabur2017exploring,bartlett2017spectrally,liang2017fisher,golowich2017size,dziugaite2017computing,wu2017towards}.
Our theorem on generalization ability also uses this idea but is more specialized to the network architecture~\eqref{eqn:architecture}.

After the initial submission of this manuscript, we became aware of concurrent work of \cite{bhojanapalli2018smoothed}, which also considered the smoothed analysis technique to solve semi-definite programs in penalty form. The mathematical techniques in our work and \cite{bhojanapalli2018smoothed} are similar, but the focus is on two distinct problems of solving semi-definite programs and quadratic activation neural networks.

\section{Preliminaries}
\label{sec:pre}
We use bold-faced letters for vectors and matrices.
For a vector $\vect{v}$, we use $\norm{\vect{v}}_2$ to denote the Euclidean norm.
For a matrix $\mat{M}$, we denote $\norm{\mat{M}}_2$ the spectral norm and $\norm{\mat{M}}_F$ the Frobenius norm.
We let $\nullspace\left(\mat{M}\right)$ to denote the left null-space of $\mat{M}$, i.e.
\[
\nullspace\left(\mat{M}\right) = \left\{\vect{v}: \mat{v}^\top \mat{M} = \mat{0}\right\}.
\] 
We use $\Sigma\left(M\right)$ to denote the set of matrices with Frobenius norm bounded by $M$ and $\Sigma_1\left(1\right)$ to denote the set of rank-$1$ matrices with spectral norm bounded by $1$.
We also denote $\mathcal{S}^d$ the set of $d \times d$ symmetric positive semidefinite matrices.


In this paper, we characterize the landscape of  over-parameterized neural networks.
More specifically we study the properties of critical points of empirical loss.
Here for a loss function $L\left(\mat{W}\right)$, a critical point $\mat{W}^*$ satisfies $\nabla L\left(\mat{W}^*\right) = 0$.
A critical point can be a local minimum or a saddle point.\footnote{We do not differentiate between saddle points and local maxima in this paper.}
If $\mat{W}^*$ is a local minimum, then there is a neighborhood $O$ around $\mat{W}^*$ such that $L\left(\mat{W}^*\right) \le L\left(\mat{W}\right)$ for all $\mat{W} \in O$.
If $\mat{W}^*$ is a saddle point, then for all neighborhood $O$ around $\mat{W}^*$, there is a $\mat{W} \in O$ such that $L\left(\mat{W}\right) < L\left(\mat{W}^*\right)$.

Ideally, we would like a loss function that satisfies the following two geometric properties.
\begin{property}[All local minima are global]\label{prop:local_global}
If $\mat{W}^*$ is a local minimum of $L\left(\cdot\right)$ it is also the global minimum, i.e., $\mat{W}^* \in \argmin_{\mat{W}}L\left(\mat{W}\right)$.
\end{property}
\begin{property}[All saddles are strict]\label{prop:saddle_strict}
At a saddle point $\mat{W}_s$, there is a direction $\mat{U}\in\mathbb{R}^{k \times d}$ such that \begin{align*}
\vectorize{\mat{U}}^\top \nabla^2L\left(\mat{W}_s\right) \vectorize{\mat{U}} < 0.
\end{align*}
\end{property}
If a loss function $L\left(\cdot\right)$ satisfies Property~\ref{prop:local_global} and Property~\ref{prop:saddle_strict}, recent algorithmic advances in non-convex optimization show randomly initialized gradient descent algorithm or perturbed gradient descent can find a global minimum~\citep{lee2016gradient,ge2015escaping,jin2017escape,du2017gradient}.

Lastly, standard applications of Rademacher complexity theory will be used to derive generalization bounds. 
\begin{defn}[Definition of Rademacher Complexity]
Given a sample $S = \left(\vect{x}_1,\ldots,\vect{x}_n\right)$, the empirical Rademacher complexity of a function class $\mathcal{F}$ is defined as \begin{align*}
R_S\left(\mathcal{F}\right) = \frac{1}{n}\expect_{\vect{\sigma}}\left[\sup_{f \in \mathcal{F}}\sum_{i=1}^{n}\sigma_i f(\vect{x}_i)\right]
\end{align*} where $\vect{\sigma} = \left(\sigma_1,\ldots,\sigma_m\right)$ are independent random varaibles drawn from the Rademacher distribution, i.e., $\prob\left(\sigma_i=1\right) = \prob\left(\sigma_i=-1\right) = 1/2$ for $i=1,\ldots,n$.
\end{defn}

\section{Overparametrization Helps Optimization}
\label{sec:opt}
In this section we present our main results on explaining why over-parametrization helps local search algorithms find a global optimal solution.
We consider two kinds of over-parameterization, $k\ge d$ and $\frac{k(k+1)}{2} > n$.
We begin with the simpler case when $k \ge d$.

\begin{thm}\label{thm:main_optimization}
Assume we have an arbitrary data set of input/label pairs $\vect{x}_i \in \mathbb{R}^d$ and $y_i \in \mathbb{R}$ for $i=1,\ldots,n$ and a convex $C^2$ loss $\ell(\hat y, y)$.
Consider a neural network of the form 
$f(\vect{x},\mat{W}) = \sum_{j=1}^{k}\relu{\vect{w}_j^\top\vect{x}}$
with $\relu{z}=z^2$ and $\mat{W} \in \mathbb{R}^{k\times d}$ denoting the weights connecting input to hidden layers.
Suppose  $k \ge d$.
Then, the training loss as a function of weight $\mat{W}$ of the hidden layers\begin{align*}
L\left(\mat{W}\right) = \frac{1}{2n}\ell\left(f(\vect{x}_i,\mat{W}),y_i\right) + \frac{\lambda}{2}\norm{\mat{W}}_F^2
\end{align*}
obeys Property~\ref{prop:local_global} and Property~\ref{prop:saddle_strict}.
\end{thm}
The above result states that given an \emph{arbitrary} data set, the optimization landscape has benign properties that facilitate finding globally optimal neural networks.
In particular, by setting the last layer to be the average pooling layer, all local minima are global minima and all saddles have a direction of negative curvature.
This in turn implies that gradient descent on the first layer weights, when initialized at random, converges to a global optimum.
These desired properties hold as long as the hidden layer is wide ($k \ge d$).

An interesting and perhaps surprising aspect of Theorem~\ref{thm:main_optimization} is its generality.
It applies to arbitrary data set of any size with any convex differentiable loss function.


Now we consider the second case when $\frac{k(k+1)}{2} > n$. 
As mentioned earlier, in practice this is often a milder requirement than $k\ge d$, and one of the main novelties of this paper.

\begin{thm}\label{thm:main_optimization_2}
	Assume we have an arbitrary data set of input/label pairs $\vect{x}_i \in \mathbb{R}^d$ and $y_i \in \mathbb{R}$ for $i=1,\ldots,n$, and a convex $C^2$ loss $\ell(\hat y, y)$.
	Consider a neural network of the form 
	$f(\vect{x}) = \sum_{j=1}^{k}\relu{\vect{w}_j^\top\vect{x}}$
	with $\relu{z}=z^2$ and $\mat{W} \in \mathbb{R}^{k\times d}$ denoting the weights connecting input to hidden layers.
	Suppose  $\frac{k\left(k+1\right)}{2} > n, k < d$  and $\mat{C}$ is a random positive semidefinite matrix with $\norm{\mat{C}}_F \le \delta$ whose distribution is absolutely continuous with respect to Lebesgue measure.
	Then, the training loss as a function of weight $\mat{W}$ of the hidden layers\begin{align}
	L\left(\mat{W}\right) = \frac{1}{2n}\ell\left(f(\vect{x}_i),y_i\right) + \frac{\lambda}{2}\norm{\mat{W}}_F^2 + \langle \mat{C}, \mat{W}^\top \mat{W}\rangle \label{eqn:perturbed_obj}
	\end{align} 
	obeys  Property~\ref{prop:local_global} and Property~\ref{prop:saddle_strict}.
Further, any global optimal solution $\widehat{\mat{W}}$of Problem~\eqref{eqn:perturbed_obj} satisfies \begin{align*}
&\frac{1}{2n}\sum_{i=1}^{n}\ell\left(f(\vect{x}_i,\widehat{\mat{W}}),y_i\right) + \frac{\lambda}{2}\norm{\widehat{\mat{W}}}_F^2 \\
\le &\frac{1}{2n}\sum_{i=1}^{n}\ell\left(f(\vect{x}_i,\mat{W}^*),y_i\right) + \frac{\lambda}{2}\norm{\mat{W}^*}_2^F + \delta \norm{\mat{W}^*}_F^2
\end{align*} where  \[
\mat{W}^*
\in \argmin_{\mat{W}}\sum_{i=1}^{n}\frac{1}{2n}\ell\left(f(\vect{x}_i,\mat{W}),y_i\right) + \frac{\lambda}{2}\norm{\mat{W}^*}_F^2. \]
\end{thm}
Similar to Theorem~\ref{thm:main_optimization}, Theorem~\ref{thm:main_optimization_2} states that if $\frac{k(k+1)}{2} >n$, then for an arbitrary data set, the \emph{perturbed} objective function~\eqref{eqn:opt_problem_perturbed} has the desired properties that enable local search heuristics to find globally optimal solution for a general class of loss functions.
Further, we can choose this perturbation to be arbitrarily small so the minimum of~\eqref{eqn:opt_problem_perturbed} is close to~\eqref{eqn:opt_problem}. 



The proof of theorem is inspired by a line of literature started by~\citet{pataki1998rank,pataki2000geometry,burer2003nonlinear,boumal2016non}. 
In summary, \citet{boumal2016non} showed that for ``almost all" semidefinite programs, every local minima of the rank $r$ non-convex formulation of an SDP is a global minimum of the original SDP. 
However, this theorem applies with the important caveat of only applying to semidefinite programs that do not fall into a measure zero set. 
Our primary contribution is to develop a procedure that exploits this by a) constructing a perturbed objective to avoid the measure zero set, b) proving that the perturbed objective has Property \ref{prop:local_global} and~\ref{prop:saddle_strict}, and c) showing the optimal value of the perturbed objective is close to the original objective. 
Further note that the analysis of \cite{boumal2016non} does not apply since our loss functions, such as the logistic loss, are not semi-definite representable.
We refer readers to Section~\ref{sec:proof_main_2} for more technical insights.

\section{Weight-decay Helps Generalization}
\label{sec:generalization}
In this section we switch our focus to the generalization ability of the learned neural network.
Since we use weight-decay or equivalently $\ell_2$ regularization in~\eqref{eqn:opt_problem}, the Frobenius norm of learned weight $\mat{W}$ is bounded.
Therefore, in this section we focus weight matrix in bounded Frobenius norm space, i.e., 
$\norm{\mat{W}}_F \in \Sigma\left(M\right)$.

To derive the generalization bound, we first recall the classical generalization bound based on Rademacher complexity bound (c.f. Theorem 2 of~\cite{koltchinskii2002empirical}).
\begin{thm}\label{thm:classical_generalization}
Assume each data point is sampled i.i.d from some distribution $\mathcal{P}$, i.e., \[
\left(\vect{x}_i,y_i\right) \sim P \text{ for }i=1,\ldots,n.
\]
We denote $S = \left\{\vect{x}_i,y_i\right\}_{i=1}^n$ and $L_{tr}\left(\mat{W}\right) = \frac{1}{n}\sum_{i=1}^{n}\ell\left(f(\vect{x}_i,\mat{W}),y_i\right)$ and $L_{te}\left(\mat{W}\right)= \expect_{\left(\vect{x},y\right) \sim \mathcal{P}}\left[\ell\left(f(\vect{x}_i,\mat{W}),y_i\right)\right]$.
Suppose loss function $\ell(\cdot,\cdot)$ is $L$-Lipschitz in the first argument, then for all $\mat{W} \in \Sigma\left(M\right)$, we have with high probability \begin{align*}
L_{te}\left(\mat{W}\right) - L_{tr}\left(\mat{W}\right) \le CL \cdot R_{S}\left(\Sigma\left(M\right)\right)
\end{align*}
where $C>0$ is an absolute constant and $R_{S}\left(\Sigma\left(M\right)\right)$ is the Rademacher complexity of $\Sigma\left(M\right)$.
\end{thm}

With Theorem~\ref{thm:classical_generalization} at hand, we only need to bound the Rademacher complexity of $\Sigma\left(M\right)$.
Note that Rademacher complexity is a distribution dependent quantity.
If the data is arbitrary, we cannot have any guarantee.
We begin with a theorem for bounded input domain.
\begin{thm}\label{thm:bounded_x_rad}
Suppose input is sampled from a bounded ball in $\mathcal{R}^d$, i.e., $\norm{\vect{x}}_2 \le b$ for some $b > 0$ and $\expect\left[\norm{\vect{x}\vect{x}^\top}_2^2\right] \le B$, then the Rademacher complexity satisfies \[R_{S}\left(\Sigma\left(M\right)\right) \le \sqrt{\frac{2b^4M^4\log d}{n}}.\]
\end{thm}
Combining Theorem~\ref{thm:classical_generalization} and Theorem~\ref{thm:bounded_x_rad} we can obtain a generalization bound.
\begin{thm} \label{thm:generalization_bounded}
Under the same assumptions of Theorem~\ref{thm:classical_generalization} and Theorem~\ref{thm:bounded_x_rad}, we have \begin{align*}
L_{te}\left(\mat{W}\right) - L_{tr}\left(\mat{W}\right) \le CLM^2b^2\sqrt{\frac{\log d}{n}}. 
\end{align*} for some absolute constant $C>0$.
\end{thm}
While Theorem~\ref{thm:bounded_x_rad} is a valid bound, it is rather pessimistic because we only assume $\vect{x}$ is bounded.
Consider the following scenario in which each input is sampled from a standard Gaussian distribution $\vect{x}_i \sim N\left(0,\mat{I}\right)$.
Then ignoring the logarithmic factors, using standard Gaussian concentration bound we can show with high probability $\norm{\vect{x}_i}_2 = \widetilde{O}\left(\sqrt{d}\right)$. \footnote{$\widetilde{O}(\cdot)$ hides logarithmic factors.}
Plugging in this bound we have \begin{align}
L_{te}\left(\mat{W}\right) - L_{tr}\left(\mat{W}\right) \le CLM^2\sqrt{\frac{d^2\log d}{n}} \label{eqn:Gaussian_worse}
\end{align}
Note in this bound, it has a quadratic dependency on the dimension, so we need to have $n = \Omega\left(d^2\right)$ to have a meaningful bound.

In fact, for specific distributions like Gaussian using Theorem \ref{thm:generalization_fourth_order}, we can often derive a stronger generalization bound. 
\begin{cor}\label{thm:Gaussian_x_rad}
Suppose $\vect{x}_i \sim N(0,\mat{I})$ for $i=1,\ldots,n$.
If the number of samples satisfies $n \ge d\log d$, we have with high probability that Rademacher complexity satisfies \[R_{S}\left(\Sigma\left(M\right)\right) \le C\sqrt{\frac{M^4d}{n}}.\]
for some absolute constant $C > 0$. 
\label{cor:gaussian}
\end{cor}
Again, combining Theorem~\ref{thm:classical_generalization} and Corollary~\ref{cor:gaussian}  we obtain the following generalization bound for Gaussian input distribution

\begin{thm} \label{thm:generalization_Gaussian}
Under the same assumptions of Theorem~\ref{thm:classical_generalization} and Corollary \ref{cor:gaussian}, we have 
\begin{align*}
L_{te}\left(\mat{W}\right) - L_{tr}\left(\mat{W}\right) \le CLM^2\sqrt{\frac{d}{n}}. 
\end{align*} for some absolute constant $C>0$.

\end{thm}
Comparing Theorem~\ref{thm:generalization_Gaussian} with generalization bound~\eqref{eqn:Gaussian_worse}, Theorem~\ref{thm:generalization_Gaussian} has an $O(\sqrt{d})$ advantage.
Theorem~\ref{thm:generalization_Gaussian} has the $\sqrt{d/n}$ dependency, which is the usual parametric rate.
Further in practice, number of training samples and input dimension are often of the same order for common datasets and architectures~\citep{zhang2016understanding}.


Corollary \ref{cor:gaussian} is a special case of the more general Theorem \ref{thm:generalization_fourth_order} which only requires a bound on the fourth moment, $\norm{\sum_{i=1}^n (x_i x_i ^\top )^2}_2 \le s$. In general, our theorems suggest if the Frobenius norm of weight matrix $\mat{W}$ is small and the input $\vect{x}$ is sampled from a benign distribution with controlled 4th moments, then we have good generalization.

As a concrete scenario, consider a favorable setting where the true data can be correctly classified by a small network using only $k_0 \ll k$ hidden units. The weights $W^*$ are non-zero only in the first $k_0$ rows and $\max_{j \in[k_0]}\norm{e_j^\top W^* }_2 \le R$. From Theorem \ref{thm:generalization_Gaussian} to reach generalization gap of $\epsilon$, we have sample complexity of $n \ge \frac{1}{\epsilon^2}C^2 L^2  R^4 dk_0 ^2 $, which only depends on the effective number of hidden units $k_0 \ll k$.  The same result can be reached for more general input distributions by using Theorem \ref{thm:generalization_fourth_order} in place of Theorem \ref{thm:generalization_Gaussian}.

\section{Proofs}
\label{sec:proof}
\subsection{Proof of Theorem~\ref{thm:main_optimization} and Theorem~\ref{thm:main_optimization_2}}\label{sec:proof_main_1}
Our proofs of over-parametrization helps optimization build upon existing geometric characterization on matrix factorization.
We first cite a useful Theorem by~\citet{haeffele2014structured}.\footnote{Theorem 2 of \cite{haeffele2014structured} assumes $\mat{W}$ is a local minimum, but scrutinizing its proof, we can see that the assumption can be relaxed to $\nabla L\left(\mat{W}\right) = \mat{0} \text{ and } \nabla^2 L\left(\mat{W}\right) \succcurlyeq \mat{0}$.}
\begin{thm}[Theorem 2 of~\cite{haeffele2014structured} adapted to our setting]\label{lem:bach}
Let $\ell\left(\cdot,\cdot\right)$ be a twice differentiable convex function in the first argument.
If the function $L\left(\mat{W}\right)$  defined in~\eqref{eqn:opt_problem} at a rank-deficient matrix $\mat{W}$ satisfies \begin{align*}
\nabla L\left(\mat{W}\right) = \mat{0} \text{ and } \nabla^2 L\left(\mat{W}\right) \succcurlyeq \mat{0},
\end{align*} then $\mat{W}$ is a global minimum.
\end{thm} 

\begin{proof}[Proof of Theorem~\ref{thm:main_optimization}]
We prove Property~\ref{prop:local_global} and Property~\ref{prop:saddle_strict} simultaneously by showing if a $\mat{W}$ satisfy\begin{align*}
\nabla L\left(\mat{W}\right) = \mat{0} \text{ and } \nabla^2 L\left(\mat{W}\right)\succcurlyeq 0
\end{align*} then it is a global minimum.

If $\rank\left(\mat{W}\right) < d$, we can directly apply Theorem~\ref{lem:bach}.
Thus it remains to  consider the case $\rank\left(\mat{W}\right) = d$.
We first notice that $\nabla L\left(\mat{W}\right) = 0$ is equivalent to
\begin{align}
\mat{W}\left(\frac{1}{n}\sum_{i=1}^{n}\frac{\partial \ell\left(\hat{y}_i,y_i\right)}{\partial \hat{y}_i}\vect{x}_i\vect{x}_i^\top + \lambda \mat{I}\right) = \mat{0} \label{eqn:grad_0_condition}
\end{align} where $\hat{y}_i =f(\vect{x}_i,\mat{W})$.
Since $\rank\left(\mat{W}\right) = d$ and $k \ge d$, we know $\mat{W}$ has a left pseudo-inverse, i.e., there exists $\mat{W}^\dagger$ such that $\mat{W}^\dagger \mat{W} = \mat{I}$.
Multiplying $\mat{W}^\dagger$ on the left in Equation~\eqref{eqn:grad_0_condition}, we have \begin{align}
\frac{1}{n}\sum_{i=1}^{n}\frac{\partial \ell\left(\hat{y}_i,y_i\right)}{\partial \hat{y}_i}\vect{x}_i\vect{x}_i^\top + \lambda \mat{I} = \mat{0}.\label{eqn:k=d,full_rank_condition}
\end{align}
To prove Theorem~\ref{thm:main_optimization}, the key idea is to consider the follow reference optimization problem.
\begin{align}
	\min_{\mat{M}: \mat{M} \in \mathcal{S}^d}L\left(\mat{M}\right) = \frac{1}{n}\sum_{i=1}^{n}\ell\left(\vect{x}_i^\top \mat{M}\vect{x}_i,y_i\right) + \lambda \norm{\mat{M}}_*. \label{eqn:opt_M}
\end{align}
Problem~\eqref{eqn:opt_M} is a convex optimization problem in $\mat{M}$ and has the same global minimum as the original problem. 
Now we denote $\tilde{y}_i = \vect{x}_i^\top \mat{M} \vect{x}_i$.
Since this is a convex function, the first-order optimality condition for global optimality is \[
\mat{0 \in }\frac{1}{n}\sum_{i=1}^{n}\frac{\partial \ell\left(\tilde{y}_i,y_i\right)}{\partial \tilde{y}_i}\vect{x}_i\vect{x}_i^\top + \lambda \partial \norm{\mat{M}}_*,
\]
$\mat{M}$ is a global minimum.

Using Equation~\eqref{eqn:k=d,full_rank_condition}, we know $\mat{M} = \mat{W}^\top\mat{W}$ achieves the global minimum in Problem~\eqref{eqn:opt_M}.
The proof is thus complete. 
\end{proof}

\subsection{Proof of Theorem~\ref{thm:main_optimization_2}}\label{sec:proof_main_2}
\begin{proof}
We first prove $L_{\mat{C}}\left(\mat{W}\right)$ satisfies Property~\ref{prop:local_global} and Property~\ref{prop:saddle_strict}.
Similar to the proof of Theorem~\ref{thm:main_optimization}, we prove these two properties simultaneously by showing if a $\mat{W}$ satisfy\begin{align}
\nabla L_{\mat{C}}\left(\mat{W}\right) = \mat{0} \text{ and } \nabla^2 L_{\mat{C}}\left(\mat{W}\right)\succcurlyeq 0 \label{eqn:first_second_order_condition}
\end{align} then it is a global minimum.
Because of Theorem~\ref{lem:bach}, we only need to show that if $\mat{W}$ satisfy condition~\eqref{eqn:first_second_order_condition}, it is rank deficient, i.e. $\rank \left(\mat{W}\right) < k$.

For the gradient condition, we have \begin{align*}
\mat{W}\left(\frac{1}{n}\sum_{i=1}^{n}\frac{\partial \ell \left(\hat{y}_i,y_i\right)}{\partial \hat{y}_i}\vect{x}_i\vect{x}_i^\top + \lambda \mat{I}\right) + \mat{W}\mat{C} = 0.
\end{align*}
For simplicity we denote $v_i = \frac{\partial \ell\left(\hat{y}_i,y_i\right)}{\partial \hat{y}_i}$ where $\hat{y}_i = \vect{x}_i^\top\mat{W}^\top\mat{W}\vect{x}_i$ and $\mat{S}\left(\vect{v}\right) = \sum_{i=1}^{n} v_i\vect{x}_i\vect{x}_i^\top$.
Using the first order condition we know $\mat{W}$ is in the null space of $\left(\mat{S}\left(\vect{v}\right)+\lambda\mat{I}+\mat{C}\right)$.
Thus, we can bound the rank of $\mat{W}$ by\begin{align*}
	\rank\left(\mat{W}\right) \le &\dim \nullspace \left(\mat{S}\left(\vect{v}\right)+\lambda\mat{I}+\mat{C}\right) \\
	\le &  \max_{\vect{v}} \dim\nullspace\left(\mat{S}\left(\vect{v}\right)+\lambda\mat{I}+\mat{C}\right).
\end{align*}
We prove by contradiction.
Assume $\rank\left(\mat{W}\right)  \ge k$, we must have \begin{align*}
k \le \max_{\vect{v}}\nullspace\left(\mat{S}\left(\vect{v}\right)+\lambda\mat{I}+\mat{C}\right).
\end{align*}
Now define $\mat{M} = \mat{S}\left(\vect{v}^*\right) + \lambda \mat{I} + \mat{C}$ with \[\vect{v}^* = \argmax \dim \nullspace\left(\mat{S}\left(\vect{v}\right) + \lambda\mat{I} + \mat{C}\right).\]
Thus we have following conditions \begin{align*}
	\mat{C} = &\mat{M} - \mat{S}\left(\vect{v}^*\right) -\lambda\mat{I}, \\
	\dim \nullspace\left(\mat{M}\right) \ge &k.
\end{align*}
The key idea is to use these two conditions to upper bound the dimension of $\mat{C}$.
To this end, we first define the set\begin{align*}
\mathcal{B}_\ell = \left\{
\mat{A}: \mat{A} = \mat{M} - \mat{S}\left(\vect{v}\right) - \lambda \mat{I}, \mat{M} \in \mathcal{S}^d, \vect{v} \in \mathbb{R}^n, \right.\\\
\left. \dim \nullspace\left(\mat{M}-\lambda \mat{I}\right) = \ell
\right\}.
\end{align*}
Note that the dimension of the manifold $\mathcal{B}_\ell$ is \begin{align*}
\left(\frac{d(d+1)}{2}-\frac{\ell\left(\ell+1\right)}{2}\right) + n
\end{align*}
where the first term is the dimension of $d \times d$ matrices, the second term is the dimension of the null space and the last term is dimension of $\range (\mat{S}\left(\vect{v}\right))$ for $\vect{v} \in \mathbb{R}^n$, which is upper bounded by $n$.

Next note that $\mathcal{B}_{\ell_1} \subset \mathcal{B}_{\ell_2}$ for $\ell_1 \ge \ell_2$.
Therefore, we can compute the dimension of the union \begin{align*}
\dim\left(\cup_{\ell=k}^d \mathcal{B}_\ell\right) = \left(\frac{d(d+1)}{2}-\frac{k(k+1)}{2}\right) + n.
\end{align*}
Note because we assume $\frac{k(k+1)}{2} > n$, we have $\dim\left(\cup_{\ell=k}^d \mathcal{B}_\ell\right) < \frac{d(d+1)}{2}$.
However, recall $\mat{C} \in \cup_{\ell=k}^d \mathcal{B}_\ell$ by definition, so we have $\mat{C}$ is sampled from a low-dimensional manifold which has Lebesuge measure $0$.
Since we sample $\mat{C}$ from a distribution that is absolute continuous with respect to the Lebesgue measure, the event $\left\{\mat{C} \in \cup_{\ell=k}^d \mathcal{B}_\ell\right\}$ happens with probability $0$.
Therefore, with probability $1$, $\rank\left(\mat{W}\right) < k$.
The proof of the first part of Theorem~\ref{thm:main_optimization_2} is complete.

For the second part.
Let $\widehat{\mat{W}} = \argmin_{\mat{W}} L_{\mat{C}}\left(\mat{W}\right)$ and $\mat{W}^{*} = \argmin_{\mat{W}}\left(\mat{W}\right)$.
Therefore we have \begin{align*}
&L\left(\widehat{\mat{W}} \right) + \langle \mat{C},\widehat{\mat{W}} ^\top\widehat{\mat{W}} \rangle \\
\le & L\left(\mat{W}^{*}\right) + \langle \mat{C},\left(\mat{W}^{*}\right)^\top\mat{W}^{*} \rangle \\
\le & L\left(\mat{W}^*\right) + \delta\norm{\mat{W}^*}_F^2.
\end{align*}
Note because $\mat{C}$ and $\widehat{\mat{W}}^\top\widehat{\mat{W}} $ are both positive semidefinite, we have $\langle \mat{C},\widehat{\mat{W}}^\top\widehat{\mat{W}} \rangle \ge 0$.
Thus $L\left(\widehat{\mat{W}}\right) \le  L\left(\mat{W}^*\right) + \delta\norm{\mat{W}^*}_F^2$.
\end{proof}

\subsection{Proof of Theorem~\ref{thm:bounded_x_rad} and Theorem~\ref{thm:generalization_Gaussian}}\label{sec:proof_generealization}
Our proof is inspired by~\citep{srebro2005rank} which exploits the structure of nuclear norm bounded space.
We first prove a general Theorem that only depends on the property of the fourth-moment of input random variables.
\begin{thm}\label{thm:generalization_fourth_order}
Suppose the input random variable satisfies $\norm{\sum_{i=1}^{n}\left(\vect{x}_i\vect{x}_i^\top\right)^2}_2 \le s$.
Then the Rademacher complexity of $\Sigma\left(M\right)$ is bounded by \begin{align*}
R_S\left(\Sigma\left(M\right)\right) \le \frac{\sqrt{2M^4s\log d}}{n}.
\end{align*}
\end{thm}

\begin{proof}
For a given set of inputs $S = \left\{\vect{x}_i\right\}_{i=1}^n$ in our context, we can write Rademacher complexity as
\begin{align*}R_{S}\left(\Sigma\left(M\right)\right) = \frac{1}{n}\expect_{\vect{\sigma}}\left[\sup_{\mat{W} \in \Sigma\left(M\right)}\sum_{i=1}^{n}\sigma_i\vect{x}_i^\top \mat{W}^\top\mat{W}\vect{x}_i\right] 
\end{align*}
Since Rademacher complexity does not change when taking convex combinations, we can first bound Rademacher complexity of  the  class of rank-$1$ matrices with spectral norm bounded by $1$ and then take convex hull and scale by $M$.
Note for $\mat{W} \in \Sigma_1\left(1\right)$, we can write $\mat{W} = \vect{v}\vect{w}^\top$ with $\norm{\vect{w}}_2 \le 1$ and $\norm{\vect{v}}_2 = 1$.
Using this expression, we can obtain an explicit formula of Rademacher complexity.
\begin{align*}
&R_{S}\left(\Sigma_1\left(1\right)\right)\\
 = &\frac{1}{n}\expect_{\vect{\sigma}}\left[\sup_{\mat{W} \in \Sigma_1\left(1\right)}\sum_{i=1}^{n}\sigma_i\vect{x}_i^\top \mat{W}^\top\mat{W}\vect{x}_i\right] \\
 = &\frac{1}{n}\expect_{\vect{\sigma}}\left[
 \sup_{\mat{W} \in \Sigma_1\left(1\right)}
 \sum_{i=1}^{n}\sigma_i\vect{x}_i^\top \left(\mat{w}^\top\vect{v}\right)^\top\left(\vect{v}\vect{w}^\top\right)\vect{x}_i
 \right] \\
 = &\frac{1}{n}\expect_{\vect{\sigma}}\left[\sup_{\vect{w}: \norm{\vect{w}}_2 \le 1 }\sum_{i=1}^{n}\sigma_i\vect{x}_i^\top \vect{w}\vect{w}^\top\vect{x}_i\right]\\
 = & \frac{1}{n}\expect_{\vect{\sigma}}\left[\sup_{\vect{w}: \norm{\vect{w}}_2 \le 1 }\sum_{i=1}^{n}\sigma_i\vect{w}^\top\vect{x}_i^\top \vect{x}_i\vect{w}\right] \\
 = &\frac{1}{n}\expect_{\vect{\sigma}}\left[\norm{\sum_{i=1}^{n}\sigma_i\vect{x}_i\vect{x}_i^\top}_{2}\right].
\end{align*}
Now, to bound  $\expect_{\vect{\sigma}}\left[\norm{\sum_{i=1}^{n}\sigma_i\vect{x}_i\vect{x}_i^\top}_{2}\right]$, we can use the results from random matrix theory on Rademacher series.
Recall that we assume \begin{align*}
\norm{\sum_{i=1}^{n}\left(\vect{x}_i\vect{x}_i^\top\right)^2}_2 \le s
\end{align*}
 and notice that \[\expect_{\sigma}\left[\sum_{i=1}^{n}\sigma_i\vect{x}_i\vect{x}_i^\top\right]=\mat{0}.\]
Applying Rademacher matrix series expectation bound (Theorem 4.6.1 of~\cite{tropp2015introduction}), we have \begin{align*}
\expect_{\vect{\sigma}}\left[\norm{\sum_{i=1}^{n}\sigma_i\vect{x}_i\vect{x}_i^\top}_{2}\right] \le &\sqrt{2 \norm{\sum_{i=1}^{n}\left(\vect{x}_i\vect{x}_i^\top\right)^2}_2\log d}\\
\le & \sqrt{2s\log d}. 
\end{align*}
Now taking the convex hull and, scaling by $M$ we obtain the desired result.
\end{proof}

With Theorem~\ref{thm:generalization_fourth_order} at hand, for different distributions, we only need to bound $\norm{\sum_{i=1}^{n}\left(\vect{x}_i\vect{x}_i^\top\right)^2}_2$.

\begin{proof}[Proof of Theorem~\ref{thm:generalization_bounded}]
Since we assume $\norm{\vect{x}_i}_2 \le b$, we directly have \begin{align*}
\norm{\sum_{i=1}^{n}\left(\vect{x}_i\vect{x}_i^\top\right)^2}_2 \le & \sum_{i=1}^{n}\norm{\left(\vect{x}_i\vect{x}_i^\top\right)^2}_2 \\
=& \sum_{i=1}^{n} \norm{\vect{x_i}}_2^2\norm{\vect{x}_i\vect{x}_i^\top}_2  \\
\le &nb^4.
\end{align*}
Plugging this bound in Theorem~\ref{thm:generalization_fourth_order} we obtain the desired inequality.
\end{proof}

\begin{proof}[Proof of Corollary \ref{cor:gaussian} and Theorem~\ref{thm:generalization_Gaussian}]
	To prove Corollary \ref{cor:gaussian}, we use Theorem \ref{thm:generalization_fourth_order} and Lemma 4.7 in \cite{soltanolkotabi2017theoretical} to upper bound $s= \norm{\sum_{i=1}^n \norm{x_i}^2 x_i x_i ^\top } $. By letting $A=I$ in Lemma 4.7, we find that 
	$$
	s \le C nd ,
	$$
	with probability at least $1- \frac{C}{d}$. This completes the proof of Corollary \ref{cor:gaussian}.
	
	Using this bound in Theorem \ref{thm:generalization_fourth_order} comletes the proof of Theorem \ref{thm:generalization_Gaussian}.
\end{proof}

\section{Conclusion and Future Works}
\label{sec:con}
In this paper we provided new theoretical results on over-parameterized neural networks.
Using smoothed analysis, we showed as long as the number of hidden nodes is bigger than the input dimension \emph{or} square root of the number of training data, the loss surface has benign properties that enable local search algorithms to find global minima.
We further use the theory of Rademacher complexity to show the learned neural can generalize well.

Our next step is consider neural networks with other activation functions and how over-parametrization allows for efficient local-search algorithms to find near global minimzers. Another interesting direction to extend our results to deeper model.

\section{Acknowledgment}
\label{sec:ack}
S.S.D. was supported by NSF grant IIS1563887, AFRL grant FA8750-17-2-0212 and DARPA D17AP00001.
J.D.L. acknowledges support of the ARO under MURI Award W911NF-11-1-0303.  This is part of the collaboration between US DOD, UK MOD and UK Engineering and Physical Research Council (EPSRC) under the Multidisciplinary University Research Initiative.

\bibliography{simonduref}
\bibliographystyle{plainnat}
\newpage
\appendix

\end{document}